\documentclass[lettersize,journal]{IEEEtran}

\usepackage{amsmath,amsfonts,bm}

\def\eqref#1{equation~\ref{#1}}

\def\1{\bm{1}}

\DeclareMathAlphabet{\mathsfit}{\encodingdefault}{\sfdefault}{m}{sl}
\SetMathAlphabet{\mathsfit}{bold}{\encodingdefault}{\sfdefault}{bx}{n}

\usepackage{amsthm}

\newtheoremstyle{nonitalictheorem}
  {\topsep}         
  {\topsep}         
  {\normalfont}     
  {}                
  {\bfseries}       
  {.}               
  {.5em}            
  {}                

\usepackage{booktabs} 
\usepackage{hyperref}
\usepackage{url}
\usepackage{amsmath}
\usepackage{dsfont}
\usepackage{amsfonts}
\usepackage{latexsym}
\usepackage{enumerate}
\usepackage{wasysym}
\usepackage{lipsum}
\usepackage{subfig}
\usepackage{graphicx}
\usepackage{float}
\usepackage{amssymb}
\usepackage{amsmath}
\usepackage{mathtools}
\usepackage{enumitem}
\usepackage{amsthm}
\newtheorem{theorem}{Theorem}
\newtheorem{assumption}{Assumption}

\newtheorem{Corollary}{Corollary}
\theoremstyle{nonitalictheorem}

\usepackage{wrapfig,commath}
\usepackage{color}
\usepackage{algorithm}
\usepackage{algorithmic}
\allowdisplaybreaks[4]
\def\BibTeX{{\rm B\kern-.05em{\sc i\kern-.025em b}\kern-.08em
T\kern-.1667em\lower.7ex\hbox{E}\kern-.125emX}}
\usepackage{balance}

\begin{document}
%
\title{Prioritizing Modalities: Flexible Importance Scheduling in Federated Multimodal Learning}

\author{Jieming Bian, Lei Wang, 
    Jie Xu,~\IEEEmembership{Senior Member,~IEEE}
\thanks{Jieming Bian, Lei Wang and Jie Xu are with the Department of Electrical and Computer Engineering, University of Florida, Gainesville, FL 32611, USA.
Email: \{jieming.bian, leiwang1, jie.xu\}@ufl.edu. 
}
}


\maketitle

\begin{abstract}
Federated Learning (FL) is a distributed machine learning approach that enables devices to collaboratively train models without sharing their local data, ensuring user privacy and scalability. However, applying FL to real-world data presents challenges, particularly as most existing FL research focuses on unimodal data. Multimodal Federated Learning (MFL) has emerged to address these challenges, leveraging modality-specific encoder models to process diverse datasets. Current MFL methods often uniformly allocate computational frequencies across all modalities, which is inefficient for IoT devices with limited resources. In this paper, we propose FlexMod, a novel approach to enhance computational efficiency in MFL by adaptively allocating training resources for each modality encoder based on their importance and training requirements. We employ prototype learning to assess the quality of modality encoders, use Shapley values to quantify the importance of each modality, and adopt the Deep Deterministic Policy Gradient (DDPG) method from deep reinforcement learning to optimize the allocation of training resources. Our method prioritizes critical modalities, optimizing model performance and resource utilization. Experimental results on three real-world datasets demonstrate that our proposed method significantly improves the performance of MFL models.
\end{abstract}




\maketitle
\section{Introduction}
\IEEEPARstart{F}{ederated}  Learning (FL) \cite{mcmahan2017communication} is a distributed machine learning approach where clients collaborate to train a machine learning model by sharing model parameters. This decentralized learning model enables devices to train models collaboratively without sharing their local, private data. Existing FL works \cite{xu2019hybridalpha, mothukuri2021survey, bonawitz2019towards, yu2019linear, yang2021achieving, bian2022mobility, bian2024accelerating} have enhanced several aspects of this method, including user data privacy, scalability to new clients and datasets, and accelerated model convergence rates. Despite these strengths, practical challenges arise when applying FL to complex, real-world data, particularly as most existing FL research focuses on unimodal data. In practical FL applications \cite{yu2021fedhar, wu2020fedhome, baghersalimi2023decentralized}, clients often comprise Internet of Things (IoT) devices, such as smartphones and unmanned aerial vehicles (UAVs), equipped with multimodal sensors capable of capturing diverse datasets.

Multimodal data, which often have varying data structures, present challenges for conventional FL methods designed for unimodal data, underscoring the need for approaches that can handle multimodal information, which is more reflective of real-world applications. Multimodal Federated Learning (MFL) has emerged in response to this need. A common approach in MFL involves employing a modality-specific encoder model for each type of data to extract features, reducing high-dimensional raw data to lower-dimensional feature representations. These features are then input into a header encoder, typically comprising deep neural network (DNN) layers with a softmax function for classification. The existing works \cite{feng2023fedmultimodal, saeed2020federated, xiong2022unified, chen2022fedmsplit, yu2023multimodal, peng2024fedmm} in MFL have different focuses. Some studies \cite{yu2023multimodal, xiong2022unified} develop methods to more effectively utilize multimodal data at the feature level to enhance prediction accuracy. Others \cite{chen2022fedmsplit, feng2023fedmultimodal} address scenarios where certain modalities are missing from some clients, or where, despite the ability to process all modalities, limited communication resources restrict the upload of updates for selected modalities to the server \cite{yuan2024communication}. 

\begin{figure}[t]
	\centering	\includegraphics[width=0.99\linewidth]{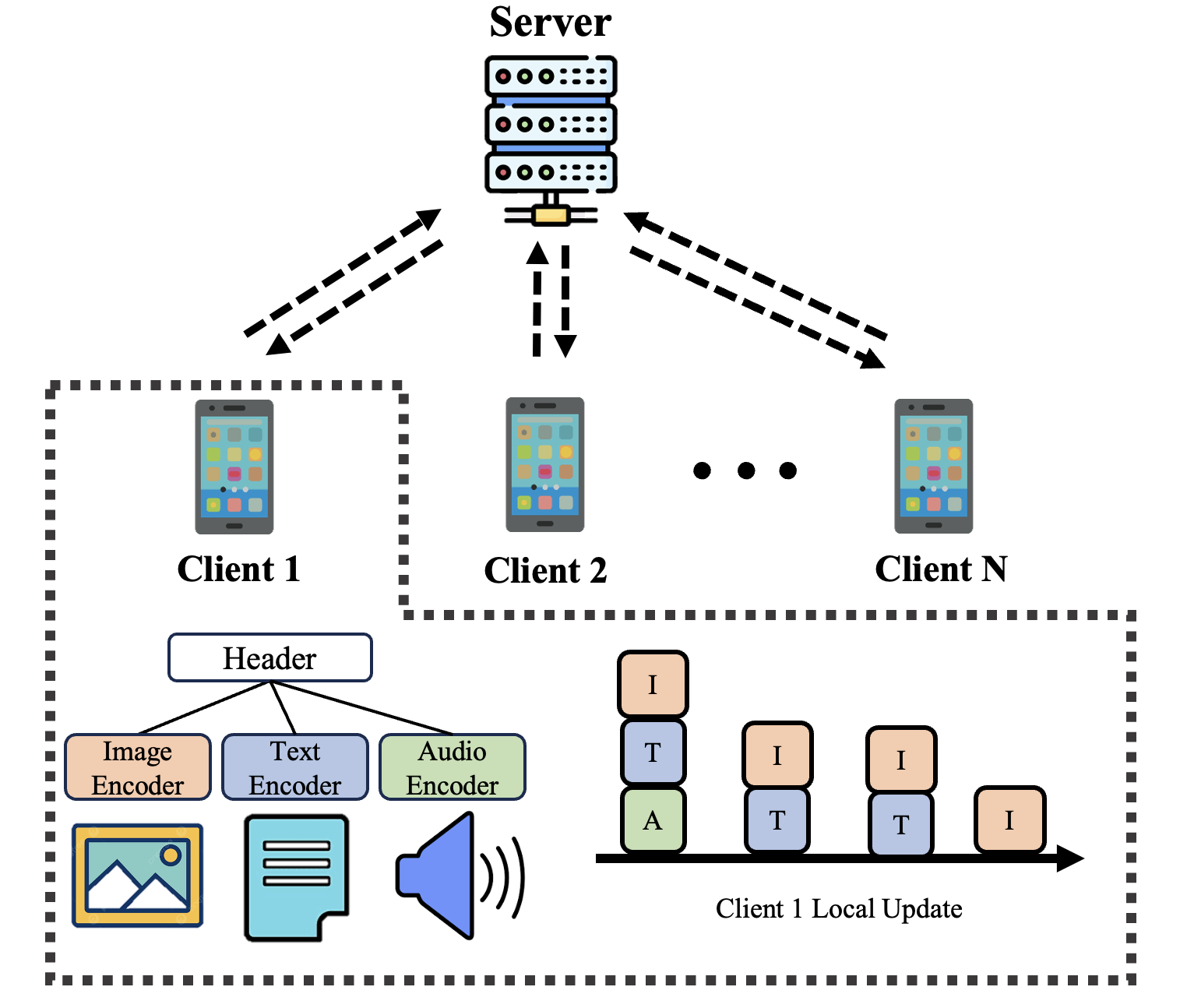}
	\caption{Modality-specific encoders extract features from each modality data, which are then input into a header encoder. Given the differing formats and varying importance of modalities, training of modality encoders should be unequal.} \label{fig: MFL}
\vspace{-0.2in}
\end{figure}

While there have been improvements in model performance, none of the existing MFL methods address the critical issue of computational resource allocation. Current MFL approaches uniformly distribute computational resources across all existing modalities' training processes. However, in real-world applications, most clients are IoT devices with limited computational resources and constrained local training timeframes \cite{imteaj2021survey, nguyen2021federated}. Efficient management of these limited resources in each training round is essential for optimizing model performance. Moreover, in the MFL setting, data comes from various modalities, each with distinct structures and dimensions, which can be significantly large. The complexity of feature extraction varies across these modalities, necessitating different types of encoder models for each. This diversity in complexity has direct implications for training time and computational resource allocation. Allocating the same training frequencies to all modality-specific encoders leads to training inefficiencies.

\subsection{A Motivational Study}
In this paper, we aim to enhance the Federated Learning framework by introducing the capability to adaptively adjust the training frequency for each modality encoder during local training epochs, thereby addressing the heterogeneous importance and training complexity of different modalities. Our motivation stems from a case study on a human activities classification task using the UCI-HAR dataset \cite{UCI}, which comprises two modalities: Accelerometer and Gyroscope data. Given the distinct information each modality provides, different encoder structures are used for different modalities: a CNN-based encoder for the Accelerometer modality and an LSTM-based encoder for the Gyroscope modality. We consider four training strategies:
\begin{itemize}
    \item \textbf{Gyro}: Only the Gyroscope data is used for training the classifier, hence the DNN model trains only the LSTM-based encoder.
    \item \textbf{Acce}: Only the Accelerometer data is used for training the classifier, hence the DNN model trains only the CNN-based encoder.
    \item \textbf{Both-Par}: Both Gyroscope and Accelerometer data are used, and the DNN model trains both encoders in parallel.
    \item \textbf{Both-Seq}: Both modality data are used, and the DNN model trains both encoders sequentially in each training round.
\end{itemize}
The experiment is conducted in an FL setting with 10 non-iid clients, each having 800 training data samples. We derive the following important insights from this experiment:

\textit{1. Varying Importance Across Modalities}: As illustrated in Fig. \ref{motivation}(a), the classification accuracy varies depending on the modalities involved in training. As expected, using both modalities to train the classification model results in the highest accuracy. However, the accuracy achieved by using data from only one modality differs depending on which modality is employed. This implies that different modalities have different importance to the training performance, and with limited computational resources, better training performance may be achieved by giving higher priority to more critical modalities.

\textit{2. Varying Training Complexities Across Modalities}: As shown in Fig. \ref{motivation}(b), the training time varies for different strategies. Training a single modality encoder is faster than training both modality encoders in parallel or sequentially. This suggests the feasibility of increasing the training frequencies for higher-priority modality encoders compared to the conventional strategy where both modalities must be trained simultaneously. Additionally, the training time for a single modality encoder varies across modalities as the encoder structures entail different complexities. Thus, allocating training resources to modality encoders must consider both the importance of the modality and its training complexity. Finally, training both modality encoders in parallel saves time compared to the sum of training times for the encoders separately or training both modality encoders sequentially. Therefore, if sufficient training resources are available, updating the encoders in parallel is still preferable.


\begin{figure}[t]
\centering
 \subfloat[Classification Accuracy]{\includegraphics[width=0.495\linewidth, height=0.33\linewidth]{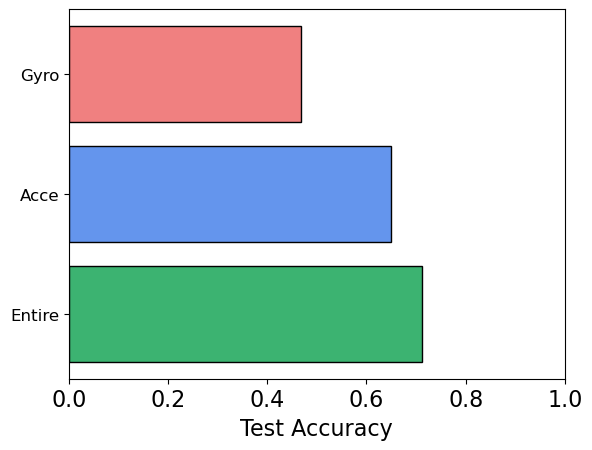}} 
\subfloat[Training Time]{\includegraphics[width=0.495\linewidth,  height=0.33\linewidth]{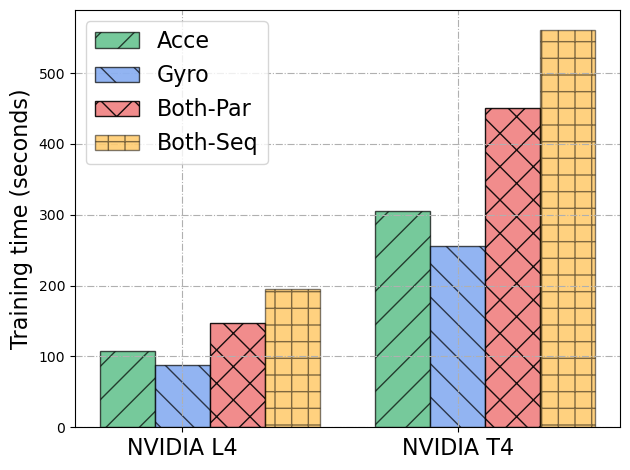}} 
\caption{Motivational Study on the UCI-HAR Dataset. Figure (a) shows the classification accuracy on the test dataset in a federated learning setting. Figure (b) displays the training time required for 20,000 updates on a single GPU (Nvidia L4 or T4), measured in seconds.} \label{motivation}
\end{figure}

\subsection{Our Contribution}
Our motivational study suggests that modalities have varying importance and corresponding encoders with varying training complexities. Furthermore, it is possible to prioritize modalities by increasing their training frequencies, potentially enhancing model accuracy. However, determining the optimal computational resource allocation for different modalities during training is challenging since directly correlating modalities with final model performance is infeasible.

In this paper, we introduce a novel, computationally efficient approach to Multimodal Federated Learning (MFL), which we refer to as FlexMod (Flexible Modality Scheduling). Our methodology comprises the following steps:

\textbf{First}, we investigate the impact of each modality based on three factors: the quality of each trained modality encoder, the importance of each modality’s features in the final classification, and the time required for training each type of encoder. Specifically, we employ prototype learning to assess the training adequacy of each modality, where prototypes serve as a form of global prior knowledge. To evaluate the importance of each modality, we utilize Shapley values \cite{shapley1953value} to quantify their respective contributions to model performance. As the motivation study demonstrates that training modalities in parallel (i.e., Both-Par) saves time compared to training modalities sequentially (i.e., Both-Seq), we extend the modality factors to modality combination factors.

\textbf{Second}, we formulate an optimization problem for each local training epoch to determine the training frequency among different modalities, considering the aforementioned factors in the objective function. In each optimization problem, we adopt the Deep Deterministic Policy Gradient (DDPG) method \cite{lillicrap2015continuous} from deep reinforcement learning (DRL) to adaptively adjust the weights of each factor in our objective function.

\textbf{Lastly}, after determining the training frequency for each modality combination, we set the training orders of the modalities in each local training epoch based on a theoretical analysis of the training performance. Our experimental results demonstrate that our proposed method achieves faster convergence compared to conventional MFL approaches and higher accuracy compared to unimodal FL approaches.


To summarize, our main findings are as follows:
\begin{itemize}
    \item From our motivational study, we observed that different modalities have varying impacts on model performance and training complexities. This led us to develop the first method to enhance computational efficiency in MFL by adaptively prioritizing the training of more important modality encoders, significantly improving overall model performance.
    \item By using prototype learning, we evaluated the training effectiveness of modality encoders. Additionally, we quantified the contribution of each modality using Shapley values. This dual evaluation guided our strategy for more effective allocation of computational resources.
    \item We integrated deep reinforcement learning to determine the optimal balance between training performance and the importance of each modality, leading to a dynamic and efficient resource allocation approach.
    \item Our method has been validated on three real-world datasets, demonstrating that prioritizing critical modalities and using our adaptive strategy significantly enhances model performance.
\end{itemize}

The remainder of this paper is structured as follows: Section~\ref{Related Work} discusses related works. Section~\ref{Pre} states the system model and mathematically formulates the problem. Section~\ref{sec:proposed} introduces the FlexMod algorithm. Section~\ref{sec:exp} presents extensive experiments to validate the FlexMod algorithm. Finally, Section~\ref{sec:con} concludes the paper.

\section{Related Works}
\label{Related Work}
\subsection{Federated Learning}
Federated Learning (FL) aims to train a global model through collaboration among multiple clients while preserving their data privacy. FedAvg \cite{mcmahan2017communication}, the pioneering work in FL, demonstrates the advantages of this approach in terms of privacy and communication efficiency by aggregating local model parameters to train a global model. A critical problem in Federated Learning is computational efficiency \cite{yang2020energy, almanifi2023communication}. Many studies aim to improve computational efficiency to address the challenges posed by resource-constrained devices. For example, \cite{yang2020energy, anh2019efficient, tran2019federated} considers resource allocation to enhance computational efficiency, \cite{taik2021data, rai2022client} selects clients based on data quality, and \cite{jiang2022model, diao2020heterofl} performs model pruning to achieve computational efficiency. However, these existing works focus on unimodal conditions. Our paper is the first attempt to optimize computational resource allocation in a multimodal federated learning setting.

\subsection{Multimodal Federated Learning}
While existing FL methods largely focus on unimodal applications, a significant number of real-world applications involve multimodal data streams. \cite{saeed2020federated} investigates FL using multi-sensory data and proposes a self-supervised learning approach to learn robust multimodal representations in FL. \cite{xiong2022unified} designs a multimodal FL framework using the cross-attention mechanism. Moreover, \cite{chen2022fedmsplit} proposes FedMSplit, which addresses the issue of missing modalities in the multimodal setting. \cite{yu2023multimodal} uses a contrastive representation-level ensemble to learn a larger server model from heterogeneous clients across multiple modalities. \cite{feng2023fedmultimodal} introduces FedMultimodal, an FL benchmark for multimodal applications.  While these methods have led to improvements in model performance, none of the existing MFL methods address the critical issue of computational resource allocation. Current MFL approaches uniformly distribute computational resources across all existing modalities' training processes. Our work aims to optimize this allocation to enhance overall training efficiency and performance.

\subsection{Federated Learning with Prototype}
The concept of a prototype in this context refers to the average feature vectors of samples within the same class. In the FL literature, prototypes serve to abstract knowledge while preserving privacy. Specifically, some approaches \cite{mu2023fedproc, tan2022fedproto, wang2024taming} focus on achieving feature-wise alignment with global prototypes to solve the issue of domain shift. \cite{tan2022federated} employs prototypes to capture knowledge across clients, constructing client representations in a prototype-wise contrastive manner using a set of pre-trained models. These works mainly apply prototypes in a single-modality federated learning setting. Our work, however, considers utilizing global prototypes to reflect the training quality of modality encoders in MFL setting.

\subsection{Federated Learning with Deep
Reinforcement Learning}
The application of Deep Reinforcement Learning (DRL) to optimize the federated learning process has been validated as a feasible approach. Amount of FL studies incorporate DRL methodologies. For instance, \cite{wang2020optimizing} employs deep Q-learning to select a subset of devices in each communication round, aiming to maximize a reward that increases validation accuracy while minimizing the number of communication rounds. \cite{mao2024joint, rjoub2022trust} use DRL to manage client selection and bandwidth allocation in wireless FL. Additionally, \cite{zhan2020experience, zheng2022exploring} design DRL-based algorithms to allocate computation resources in FL. In this paper, we propose a DRL method to determine the weighting factor for modality combination quality.

\section{Preliminaries}
\label{Pre}
\subsection{System Model}
To clarify our problem, we begin by presenting essential notations, as outlined in Table \ref{table:Essential Notations}.

\begin{table}
\caption{Essential Notations}
\label{table:Essential Notations}
\centering
\begin{tabular}{lll}
    \toprule
    Notation     & Definition   \\
    \midrule
$M$        & Number of modality.  \\
$N$      & Number of client.  \\
$K$ & Number of label.\\
$S$        & Number of modality combinations, $S=2^M - 1$.  \\
$r$ & Index of local training period. \\
$E_r$ & The number of local updates in local period $r$. \\
$\theta^m$        & Modality $m$'s encoder model.  \\
$\theta^0$        & Header encoder model.  \\
$\Theta$        & Entire model, $\Theta = \{\theta^0, \theta^1, \dots, \theta^M\}$.  \\
$t_s$ & Unit training time for combination $s$, $s \in \{1, \dots, S\}$. \\
$T$ & Available training time for each local training. \\
$a_s^r$  & Frequency of combination $s$ during $r$-th local training. \\
$\boldsymbol{a}^r$  & Frequencies of all combinations during $r$-th local training.\\
$p_{n,k}^m$ & Local prototype of client $n$ on modality $m$ with label $k$.\\
$p_{k}^m$ & Global prototype on modality $m$ with label $k$.\\
$\omega^m$ & Encoder quality index of modality $m$.\\
$\Omega^s$ & Encoder quality index of modality combination $s$.\\
$\gamma^m$ & Importance index of modality $m$. \\
$\Gamma^s$ & Importance index of modality combination $s$. \\
$U^r$ & Utility value during the local training period $r$. \\
$\beta^r$ & Sum weight between importance and encoder quality.\\
$\mathcal{s}^r$ & State of local training period $r$ in DRL.\\
$\mathcal{A}^r$ & Action of local training period $r$ in DRL.\\
\bottomrule
\end{tabular}
\end{table}

We consider a multimodal federated learning (MFL) system encompassing $N$ clients and a single server. The objective is to enable cooperative training across these clients on a unified global model using multimodal data, with a focus on maintaining data privacy. We specifically consider a scenario where every client, denoted as $n$ for $n \in \{1, \dots, N\}$, possesses data of $M$ modalities without any missing modality, represented as $\mathcal{D}_n = (X_n^1, X_n^2, \dots, X_n^M; Y_n) = \{(x_{n, i}^1, x_{n, i}^2, \dots, x_{n, i}^M; y_{n, i})\}_{i=1}^{|\mathcal{D}_n|}$. Here, $x_{n,i}^m$ indicates the $m^{th}$ modality data of the $i^{th}$ sample in client $n$, and $y_{n,i}$ denotes the corresponding label. Without loss of generality, in this paper we consider each client $n$ contains the same amount of training data samples.

The architecture shared among the clients incorporates modality encoders ${\theta^1, \dots, \theta^M}$ to distill feature-level information from the inputs and a header encoder $\theta^0$ that integrates these features to predict labels. Considering the diversity of modalities, the modality encoders can have different structures. This diversity underscores the varying complexities inherent in training across different modalities. The feature vector for the $m^{th}$ modality of client $n$, extracted by encoder $\theta^m$, is denoted as $Z_n^m = \theta^m(X_n^m)$. The header encoder $\theta^0$ then synthesizes the outputs $Z_n^m$ from each modality encoder to forecast the labels $\hat{Y}_n$ for client $n$'s data:

\begin{align}
\hat{Y}_n = \theta^0(Z_n^1, \dots, Z_n^M),
\end{align}
where $Z_n^m =\theta^m(X_n^m)$ is the feature vectors extracted of modality $m$ at client $n$. We further denote $\Theta = \{\theta^0, \theta^1, \dots, \theta^M\}$ to represent the entire model.

The global objective is framed as an optimization problem aiming to minimize the loss function $f$, which evaluates the discrepancy between the predicted and actual labels across all clients:

\begin{align}
\text{min}_{\theta^0, \theta^1, \theta^2, \dots, \theta^M} = \sum_{n=1}^N \frac{1}{N} f(\hat{Y}_n, Y_n).
\end{align}

\subsection{Problem Formulation}
We consider that during each local training period \(r\), the maximum training time for each client is constrained to \(T\). Consequently, this limitation determines the number of local updates possible for each client within the specified period. Contrary to prior studies, which consistently train the entire model (i.e., \(\Theta\)) concurrently, we focus on a critical oversight: the distinct architectural designs of encoders necessitate varying levels of training resources and present different challenges.

We assume that all clients have comparable computational capabilities, allowing them to conclude their training sessions in a nearly synchronous manner. This setting aligns with existing MFL frameworks, where each client communicates with the central server only after their local training. Training for each encoder may occur simultaneously or separately. As our motivation study reveals, the time spent training modality combinations is less than the sum of training the modalities separately. We focus on the computational resource allocation for each modality combination rather than for individual modalities. With $M$ distinct modalities available, the number of possible training combinations is $2^M - 1$, considering the inclusion or exclusion of each modality $m$ (excluding the case where no modalities are chosen). For convenience, we denote this total number of combinations as $S = 2^M - 1$ and the set of modality included in the combination $s$ as $\mathcal{C}_s$. The variable $a_s^r$ represents the frequency of selecting combination $s$ during the $r$-th local training session. It is presumed that $a_s^r$ is uniform across all clients in the system and is a non-negative integer. The vector $\boldsymbol{a}^r = (a^r_1, \ldots, a^r_S)$ summarizes the choices of modality combinations for the $r$-th local training period. Furthermore, we denote $t_s$ as the time cost for training each combination $s$ once. As we consider the training data sizes are the same among each client, $t_s$ stands for all clients' training of combination $s$. Each unit training time $t_s$ for combination $s$ is determined and measured before training begins, thus considered a known value in advance. Consequently, for the local training period $r$, the total computational time dedicated to training combination $s$ equals $t_s \cdot a^r_s$. 

The computational constraints of each client limit their ability to train the complete model (comprising the header \(\theta^0\) and the encoders \(\theta^1, \dots, \theta^M\)) as comprehensively as desired. Previous methods have aimed to train the entire model simultaneously, maintaining uniform training frequencies across all modalities for each decision period \(r\). However, this approach of uniform training frequency allocation may not optimize efficiency. The goal is to identify the optimal modality allocation for each round \(r\) to enhance overall FL performance. While the allocation of training resources for modalities in MFL is similar to the client selection problem in classical FL—which has been effectively addressed by using deep reinforcement learning to select clients for training—this method cannot be directly applied to our context. The reason is that client selection is a binary decision: a client is either selected (i.e. 1) or not (i.e. 0), resulting in a relatively manageable action space for deep reinforcement learning applications. Conversely, in the modality computation resource allocation challenge, we focus on combinations of modalities, increasing the scope from \(M\) modalities to \(S = 2^M - 1\) possible combinations. Furthermore, the number of training iterations for each modality combination in each local training period \(r\) is not binary. Instead, for modality combination \(s\), its potential training iterations during period \(r\) could range from \(0\) to \(\lfloor\frac{T}{t_s}\rfloor\), where \(T\) is the total available local training time. Thus, if one were to apply a reinforcement learning approach directly, the action space could potentially expand to \(\prod_{s = 1} ^S \lfloor\frac{T}{t_s}\rfloor\), which becomes impractically large, especially since realistic FL settings require extended local training periods to minimize communication overhead. Consequently, developing an effective method to allocate modality combinations for training to achieve optimal performance in an MFL setting remains an unexplored challenge in existing literature.

\section{Proposed Method}
\label{sec:proposed}
To address the challenge of adaptively allocating modality combination training resources effectively within each local training period, our proposed method, FlexMod, contains three main steps: 1. We quantify the factors of each modality combination that impact the resource allocation decision. 2. We utilize the determined modality combination factors to formulate an optimization problem for each local training epoch. 3. We provide the training orders of the modality combinations in each local epoch and offer corresponding theoretical analysis. In this section, we will introduce each step in detail.

\subsection{Quantifying Modality Combination Factors}

\subsubsection{Encoder Quality}
Due to variations in data structures across different modalities, the corresponding modality encoders differ. Consequently, training difficulties also vary, naturally leading to the conclusion that encoders which are more challenging to train should receive greater attention. These encoders should be allocated more computational resources, allowing for more frequent training sessions within a single local training period compared to those required by simpler modality encoders. Therefore, the first element we consider involves assessing the quality of the modality encoder (i.e. how well the modality encoder has been trained) to impact the necessity and frequency of training.

A simplistic approach might be to estimate training progress based solely on the size of the encoder model, assuming that larger models progress more slowly and therefore require more resources. However, this method is rigid as it only assesses the difficulty of training modality encoders at the start and does not allow for adaptive reallocation of resources based on actual progress. Another intuitive method involves monitoring the differences, such as the L2 Norm, between consecutive updates of each encoder model. However, this approach is inequitable due to the structural differences between modality encoders. To address the challenge of quantifying training quality across differently structured modality encoders, we propose a solution inspired by prototype learning. Prototypes encapsulate the semantics of data within a given class.


At each client, denoted as $n$, local prototypes for each class within each modality are generated. We denote by $k \in \{1, \dots, K\}$ the label classes for each training sample at each client. Correspondingly, we define $X_{n,k}^m$ as the set of training samples of class $k$ and modality $m$ at client $n$. Therefore, we introduce the local prototype $p_{n,k}^m$ as the average value of the features extracted by the modality encoder:

\begin{align}
    p_{n,k}^m = \frac{1}{|X_{n,k}^m|}\sum_{i \in {\mathds{1}_{y_{n,i} = k}}} \theta^m(x_{n,i}^m),
\end{align}
where $x_{n,i}^m$ is a data sample from $X_{n,k}^m$. To ensure fairness among prototypes for different modalities, it is essential to observe the following guidelines for prototype generation: 1. Local prototypes across various modalities should adhere to the same architecture, meaning that the structure of outputs (i.e. features extracted $\theta^m(x_{n,i}^m)$) through modality encoders should remain consistent despite inherent differences in the original data structures of each modality. By processing the data through these encoders, the extracted features of each modality can be structurally aligned. 2. The local prototypes of different modalities should be normalized to bring all prototypes across modalities to the same magnitude.

Each client generates local prototypes for each modality after every training session and then transmits them to the server. Although this process incurs additional communication costs, these are relatively minor compared to the updates of the model and can therefore be considered negligible. Upon receiving the local prototypes for each modality, the server creates a global prototype for modality \( m \) and class \( k \) as follows:

\begin{align}
    p_k^m = \frac{1}{N}\sum_{n = 1}^N p_{n,k}^m.
\end{align}

After the server generates global prototypes for each modality \( m \), it uses these prototypes to assess the quality of the modality encoder. The rationale for this approach is that the distinctiveness of global prototypes reflects the effectiveness of the modality encoders. Specifically, the more distinguishable the global prototypes are for different classes within a modality, the more effectively the modality encoder can extract distinct features, indicating higher quality. Conversely, if the global prototypes for different classes within a modality overlap significantly, it suggests that the modality encoder requires further training to achieve optimal performance.

To quantify the degree of distinction or overlap among the global prototypes, we propose using cosine similarity. The cosine similarity metric \( \cos(p_{k}^m, p_{o}^m) \) between two prototypes \( p_{k}^m \) and \( p_{o}^m \) within modality \( m \) is calculated as follows:

\begin{align}
    \cos(p_{k}^m, p_{o}^m) = \left( \frac{p_{k}^m}{\|p_{k}^m\|} \cdot \frac{p_{o}^m}{\|p_{o}^m\|} \right),
\end{align}

We then define the modality encoder quality index \( \omega^m \) for each modality \( m \) as:

\begin{align}
    \omega^m = \frac{1}{K}\sum_{k \neq o}\cos(p_{k}^m, p_{o}^m).
\end{align}

Note that a lower \( \omega^m \) value indicates well-separated global prototypes, suggesting that the modality encoder for modality \( m \) is well-trained and thus requires fewer computational resources. Conversely, a higher \( \omega^m \) value indicates poorly separated global prototypes, signifying that the encoder is undertrained and necessitates additional training.

We then normalize the modality encoder quality as follows:
\begin{align}
\omega^m = \frac{\omega^m}{\sqrt{\sum_{i} (\omega^i)^2}}
\end{align}

By holding each individual modality's encoder quality index, and with our goal to allocate computational resources to each modality combination, we present the modality combination encoder quality index $\Omega^s$ as:
\begin{align}
    \Omega^s = \sum_{m \in \mathcal{C}_s} \omega^m,
\end{align}
where $\mathcal{C}_s$ is the set of modalities in modality combination $s$.

\subsubsection{Importance}
In the allocation of computational resources, it is crucial to consider not only the proficiency of modality encoders (how well an encoder is trained) but also the significance of each modality (how much a modality contributes to the final prediction). The concept of selecting specific modalities or features has been extensively explored in previous research \cite{cheng2022greedy}. In this paper, we propose using the Shapley Value to capture the importance of each modality.

To compute the Shapley Value for each modality, we assume that the server has access to a limited set of validation data, approximately 1\% of the total training data held by clients. The server then employs the pre-trained modality encoders, including both the modality-specific encoder \(\theta^m\) for each modality \(m\) and the header \(h(\cdot)\), to calculate the Shapley Value of modality \(m\) as follows:

\begin{align}
    \gamma^m = &\sum_{\mathcal{S} \subset \mathcal{Z} \setminus \{Z_{\text{val}}^m\}} \frac{|\mathcal{S}|!(M-|\mathcal{S}|-1)!}{M!} \nonumber\\
    &\cdot\left( f\left(\theta^0\left(\mathcal{S}\cup \{Z_{\text{val}}^m\}\right), Y_{\text{val}}\right) - f\left(\theta^0(\mathcal{S}), Y_{\text{val}}\right) \right),
\end{align}
where $\mathcal{Z} = \{Z^1_{\text{val}}, \dots, Z_{\text{val}}^M \}$ and each $Z^m_{\text{val}} = \theta^m(X^m_{\text{val}})$, with $X^m_{\text{val}}$ representing the feature input from modality $m$ in the validation dataset. Moreover, $\mathcal{S}$ represents a subset of $\mathcal{Z}$. It is important to note that in our framework, the header encoder integrates all extracted modality features to generate the predicted soft labels. Consequently, in the formula, $\theta^0(\mathcal{S})$ indicates that for the extracted feature $Z^m$ included in $\mathcal{S}$, $Z^m$ is utilized, while for those not included, zero tensors are used to represent missing features.

In the calculations of the Shapley Value outlined above, we observe that \(f(\cdot)\) represents our predetermined loss function, specifically the cross-entropy loss. Previous research \cite{baltruvsaitis2018multimodal} indicates that incorporating additional modalities generally enhances the predictive performance, suggesting that the Shapley Value, \( \gamma^m \), for a given modality \(m\) should be negative. Additionally, the range of \( \gamma^m \) varies across different modalities.

To integrate the Shapley Value with the Modality Encoder Quality Index, \( \omega^m \), which lies within the interval [0,1], normalization of the Shapley Value is necessary. We normalize Shapley Value, \( \gamma^m \), as follows:
\begin{align}
\gamma^m = \frac{\gamma^m}{\sqrt{\sum_{i} (\gamma^i)^2}}
\end{align}
This normalization ensures that \( \gamma^m \) ranges from 0 to 1. As the value of \( \gamma^m \) increases, it signifies that the impact of a modality on the final prediction becomes more pronounced. Consequently, modalities with higher \( \gamma^m \) values should be prioritized due to their greater influence on capturing accurate modality features.

By maintaining each individual modality's importance index $\gamma^m$, and with our goal to allocate computational resources to each modality combination, we present the modality combination importance index $\Gamma^s$ as:
\begin{align}
    \Gamma^s = \sum_{m \in \mathcal{C}_s} \gamma^m,
\end{align}
where $\mathcal{C}_s$ is the set of modalities in modality combination $s$.

\subsubsection{Training Time}
Another factor that could impact resource allocation is the training time cost of each modality combination \(s\), which we have previously defined as \(t_s\). The value of \(t_s\) should be obtained before the FL training step.

\subsection{Formulating the Optimization Problem }

With the established values of $\Gamma^s$ and $\Omega^s$, we can define the optimization problem for the server to optimize resource allocation for each client during each training period $r$. The utility function for this period, $U^r$, is defined as:
\begin{align}
    U^r(\boldsymbol{a}^r) = \sum_{s=1}^S(\beta^r\Omega^s + (1-\beta^r) \Gamma^s) \cdot a_s^r
\end{align}
where $\boldsymbol{a}^r = (a_1^r, \cdots, a_S^r)$ and $\beta^r$ represents the weighting factor in this sum during the local period $r$. The value of $\beta^r$ will be discussed in further detail subsequently. The computational resources expended under the allocation $\boldsymbol{a}^r$ are given by:
\begin{align}
    \mathcal{T}(\boldsymbol{a}^r) = \sum_{s=1}^S t_s \cdot a_s^r
\end{align}
where $t_s$ represents the time required to train each modality combination $s$ once using local data. Thus, we formalize the computational resource allocation optimization problem as:

\begin{align}
\textbf{P1:}\quad \max & \quad U^r(\boldsymbol{a}^r)\nonumber\\
\text{s.t.} & \quad \mathcal{T}(\boldsymbol{a}^r) \leq T, \nonumber\\
& \quad a_s^r \in \mathbb{Z},
\end{align}

\textbf{Remark:} The feasibility of employing Equations (8) and (11), which uses the sum of indices from individual modalities as a combined modality index, merits discussion. Let us consider two modalities, \(m_1\) and \(m_2\), as examples. The indices for modality \(m_1\) alone are denoted as \(\omega^{m_1}\) and \(\gamma^{m_1}\), and similarly, for \(m_2\) as \(\omega^{m_2}\) and \(\gamma^{m_2}\). When \(m_1\) and \(m_2\) are combined, the indices are summed: \(\omega^{m_1} + \omega^{m_2}\) and \(\gamma^{m_1} + \gamma^{m_2}\). This approach suggests that training \(m_1\) and \(m_2\) separately results in utility values equivalent to training them together, represented mathematically as \(U^r((a_{\{m_1\}}^r = 1, a_{\{m_2\}}^r = 1, a_{\{m_1, m_2\}}^r = 0)) = U^r((a_{\{m_1\}}^r = 0, a_{\{m_2\}}^r = 0, a_{\{m_1, m_2\}}^r = 1))\). Notably, the time cost of training modalities together is less than training each separately. Therefore, prioritizing combined training over separate training not only conserves time but also minimizes the reliance on outdated information. This aligns with the rationale that integrating modalities more comprehensively, rather than in isolation, enhances efficiency.

In our utility function, Equation (12), we integrate the importance of modality combinations $\Gamma^s$ and the quality of modality combinations $\Omega^s$ through a weight $\beta^r$ in each local training period $r$. We conducted a preliminary experiment to demonstrate the significance of carefully determining the value of $\beta^r$. We considered three fixed values of $\beta^r$ for each round $r$: 0, 0.5, and 1. If $\beta^r$ is fixed at 1 across all training rounds, our utility function only considers the quality indices of the encoder, ignoring the importance indices. Conversely, a value fixed at 0 means only the importance indices are considered. The results shown in Fig. \ref{beta}(a) indicate that different $\beta^r$ values significantly impact the allocation decision. Furthermore, Fig. \ref{beta}(b) demonstrates that varying $\beta^r$ affects the final federated learning convergence performance. Therefore, setting an appropriate $\beta^r$ is crucial. Moreover, since the quality and importance indices change during the training process, a fixed $\beta^r$ value is not efficient. Hence, an adaptive $\beta^r$ over time is necessary.

\begin{figure}[t]
\centering
 \subfloat[Impact on the allocation decision]{\includegraphics[width=0.495\linewidth]{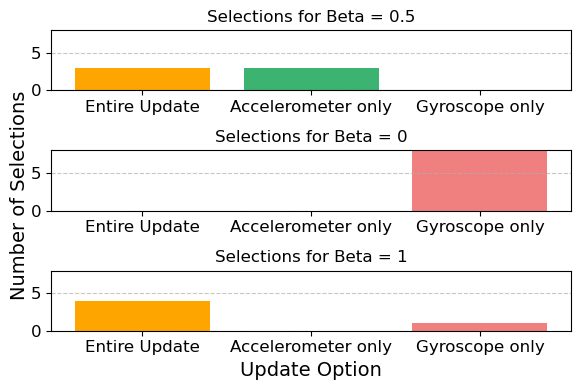}} 
\subfloat[Impact on FL performance]{\includegraphics[width=0.495\linewidth]{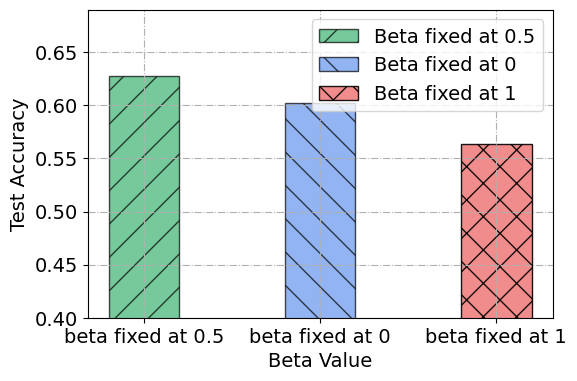}} 
\caption{The weight can impact not only the allocation decision for a given round but also the final convergence performance. Figure (a) shows that at a given round, if beta varies between 0, 0.5, and 1, the combination allocation decision could be different. Figure (b) shows that with different values of fixed weight, the final model achieves different performance levels.} \label{beta}
\end{figure}

We formalize the determination of the weighting factor for modality combination quality, denoted as $\beta^r$, as a Deep Reinforcement Learning (DRL) problem characterized by a continuous action space. To structure the DRL framework, we initially delineate the state, action, and reward components relevant to our study.

\textbf{State:} The state in round \( r \) is denoted by \( \mathcal{s}^r \) and comprises two primary elements. To minimize the risk of including redundant information, we consider states at the modality level rather than the modality combination level. Each element is represented as a vector: the modality importance vector \(\hat{\gamma} = (\gamma^1, \dots, \gamma^M)\) and the modality quality vector \(\hat{\omega} = (\omega^1, \dots, \omega^M)\). Thus, the state at round \( r \) is expressed as \( \mathcal{s}^r = (\hat{\gamma}, \hat{\omega}) \). The values \(\gamma^m\) and \(\omega^m\) for each modality \( m \) at round \( r \) are derived from data aggregated from each client before the start of round \( r \).

\textbf{Action:} The agent at the server side has task to determine the weighting factors $\beta^r$ in each round $r$. The action space is continuous, subject to the following constraints:
\begin{align}
\label{action}
    \beta^{r} \in [0,1], \quad \forall r
\end{align}

\textbf{Reward:} In accordance with the reward structure detailed in \cite{wang2020optimizing}, the reward at the end of each round $r$, denoted as $R_r$, is defined as $R_r = \phi^{\text{Acc}^r - \text{Acc}^{\text{target}}} - 1$. Here, $\text{Acc}^r$ represents the testing accuracy of the global model $\Theta^r$ on the held-out validation set after round $r$, and $\text{Acc}^{\text{target}}$ is the target accuracy at which the training in federated learning will cease once the model achieves this accuracy on the validation dataset. The parameter $\phi$ is a positive constant that ensures the reward increases exponentially with the testing accuracy. As described in \cite{wang2020optimizing}, the reward $R_r$ ranges from -1 to 0, corresponding to accuracies between 0 and 1. When federated learning terminates upon reaching the target accuracy, the reward reaches its maximum value of 0.

Compared to directly selecting modality combinations for training using DRL methods, which carries the risk of an unsolvable problem due to the excessively large action space, solving the new DRL problem for determining the weights $\beta^r$ can be achieved with relative ease, as it aligns with typical DRL problems with continuous action spaces. We employ a classical method known as the Deep Deterministic Policy Gradient (DDPG) algorithm to address this new DRL problem.

The DDPG algorithm employs a complex network architecture that consists of four neural networks: the current Actor network, the target Actor network, the current Critic network, and the target Critic network. These networks are designed with distinct roles to ensure continuous and stable learning in environments characterized by continuous action spaces. Both the current and target Actor networks follow an identical structural framework, where states are inputs and actions are the outputs. Specifically, the current Actor network processes the existing state \( \mathcal{s}^r \), generating the optimized action \( \mathcal{A}^r \) for this state. In contrast, the target Actor network receives the anticipated subsequent state \( \hat{\mathcal{s}}^{r+1} \) as input and produces the action \( \hat{\mathcal{A}}^{r+1} \), considered optimal for that future state. This approach enables the continual adaptation of actions for both present and anticipated states, thus improving the algorithm’s predictive accuracy. Likewise, the current and target Critic networks, which are structurally identical, assess the value of taking particular actions in given states. The current Critic network processes the current state \( \mathcal{s}^r \) and action \( \mathcal{A}^r \), outputting the Q-value \( Q(\mathcal{s}^r, \mathcal{A}^r) \), which quantifies the expected utility of the action at that state. Meanwhile, the target Critic network processes the predicted next state \( \hat{\mathcal{s}}^{r+1} \) and action \( \hat{\mathcal{A}}^{r+1} \), and provides a Q-value \( Q(\hat{\mathcal{s}}^{r+1}, \hat{\mathcal{A}}^{r+1}) \) that reflects the anticipated utility of the action in the upcoming state. The training of the current Critic network is facilitated by computing the target Q-value using the formula:
\begin{align}
    \text{Target Q-value} = R_r + \lambda Q(\hat{\mathcal{s}}^{r+1}, \hat{\mathcal{A}}^{r+1})
\end{align}
where \( R_r \) is the reward obtained in round \( r \), and \( \lambda \) is the discount factor, which prioritizes future rewards over immediate ones. This target Q-value is crucial for guiding the minimization of the loss between the predicted Q-value \( Q(\mathcal{s}^r, \mathcal{A}^r) \) by the current Critic network and the target Q-value, thereby enhancing policy improvement and robust value estimation across continuous action domains.

In summary, the two current networks (Actor and Critic) operate on the present state or action, whereas the two target networks (Actor and Critic) deal with the subsequent state or action. By storing transitions \( (\mathcal{s}^r, \mathcal{A}^r, R^r, \hat{\mathcal{s}}^{r+1}) \), we utilize experience replay technology to update the DDPG networks, thus enhancing learning efficiency.

\subsection{Organizing the Training Order}
With the DDPG agent, we can ascertain the weighting factors $\beta^r$ during each local training period $r$. Given $\beta^r$, solving the convex optimization problem \textbf{P1} becomes straightforward, as its convex nature is well-documented in the literature \cite{boyd2004convex}. Consequently, by resolving \textbf{P1} on the server side before each round $r$, we can derive the computational allocation decision vector $\boldsymbol{a}^r = (a^r_1, \ldots, a^r_S)$, which outlines the selected modality combinations for the $r$-th local training period. The remaining challenge in our proposed method is to organize the training sequence for the chosen combinations where $a^r_s \neq 0$.

To design the training sequence, we consider the training frequency of each combination, represented by $\boldsymbol{a}^r$, for each round. We begin with a theoretical analysis to aid in organizing the training sequence. We denote the total number of updates in the local training period $r$ as $E_r = \sum_{s=1}^S a_s^r$. For each update $e = 1, \dots, E_r$, the combination being trained is denoted by $\boldsymbol{s}_e$. Our analysis is based on the following assumptions:

\begin{assumption}
\label{ass1}
\emph{
There exist positive constants $L < \infty$, for any client $n \in [N]$, any modality $m \in [M]$, such that for all $\Theta$ and $\Theta'$, the objective function satisfies:
\begin{align}
    \lVert \nabla_m f_{n}(\Theta) - \nabla_m f_{n}(\Theta') \rVert &\leq L \lVert \Theta - \Theta' \rVert.
\end{align}
}
\end{assumption}

\begin{assumption}
\label{ass2}
\emph{
There exists a constant $\delta$ such that the squared Euclidean norm of $\nabla_{m} f_n(\Theta)$ is uniformly bounded:
\begin{align}
\left\|\nabla_{m} f_n( \Theta)\right\|^2 \leq \delta^2.  
\end{align}
}
\end{assumption}

We consider that starting from the initial entire global model $\Theta^r$ at round $r$. The optimal of the global model after the local training round $r$ could be denoted as $ \hat{\Theta}^{r+1}$, which should be all modality updates simultaneously during all local updates $E_r$. However, we know that due the computational resource limitations, such all modality updates can not be achieved. Thus, the real entire global model after the local training period $r$ should $\Theta^{r+1}$, in which in the local updates $e = 1, \dots, E_r$, only the modalities in combination $\boldsymbol{s}_e$ being trained.

\begin{theorem}
\label{thm1}
Suppose Assumptions \ref{ass1}--\ref{ass2} hold. Starting from the same initial global model \(\Theta^r\), the difference between the real entire global model $\Theta^{r+1}$ and the optimal global model $\hat{\Theta}^{r+1}$ at round \( r+1 \) is bounded as follows:
\begin{align}
    & \|\Theta^{r+1} - \hat{\Theta}^{r+1}\|^2  \nonumber\\
    \leq & 2\eta^2 \sum_{e=1}^{E_r} \left( \prod_{j=e}^{{E_r} -1} (2 + 2 \eta^2 L^2 |\mathcal{C}_{\boldsymbol{s}_j}|) \right) (M - |\mathcal{C}_{\boldsymbol{s}_e}|) \delta^2,
\end{align}
where $\eta$ is the learning rate, and $|\mathcal{C}_{\boldsymbol{s}_e}|$ represents the number of modalities selected to update in the local update $e$.
\end{theorem}


\begin{proof}
For any entire model $\Theta$, we denote that $\nabla_{\boldsymbol{s}_e} f_n(\Theta)$ as the $e$-th updates as:
\begin{align}
\nabla_{\boldsymbol{s}_e} f_n(\Theta) &= \begin{bmatrix} G_1 \\ G_2 \\ \vdots \\ G_M \end{bmatrix}
\end{align}
\text{where}
\begin{align}
G_m = \begin{cases} 
(\nabla_{m} f_n(\Theta))^{\top} & \text{if } m \in \boldsymbol{s}_e \\
0 & \text{otherwise}
\end{cases}
\end{align}
On the contract, for the optimal case, all modality being trained, we utilize  $\nabla f_n(\Theta)$ to represent.

Thus, we start with analyzing the difference between the real global model and the optimal global model which are:
\begin{align}
    & \|\Theta^{r+1} - \hat{\Theta}^{r+1}\|^2 \leq \frac{1}{N}\sum_{n=1}^N\|\Theta_n^{r,{E_r} } - \hat{\Theta}_n^{r,{E_r} }\|^2
\end{align}

Note that for,
\begin{align}
    \|\Theta_n^{r,1} - \hat{\Theta}_n^{r,1}\|^2 = & \eta^2 \sum_{m \notin \mathcal{C}_{\boldsymbol{s}_1}}\|\nabla_{m} f_n({\Theta}^r)\|^2 \nonumber\\
    \leq & \eta^2 (M - |C_{s_1}|) \delta^2
\end{align}
and we have
\begin{align}
    &\|\Theta_n^{r,e+1} - \hat{\Theta}_n^{r,e+1}\|^2 \nonumber\\
    = & \|\Theta_n^{r,e}  - \hat{\Theta}_n^{r,e} + \eta\nabla_{\boldsymbol{s}_{e+1}} f_n(\Theta_n^{r,e})- \eta\nabla f_n(\hat{\Theta}_n^{r,e})\|^2 \nonumber\\
   \leq  & 2 \|\Theta_n^{r,e} - \hat{\Theta}_n^{r,e}\|^2 + 2 \eta^2\|\nabla_{\boldsymbol{s}_{e+1}} f_n(\Theta_n^{r,e})- \nabla f_n(\hat{\Theta}_n^{r,e})\|^2 \nonumber\\
    \leq & 2 \|\Theta_n^{r,e} - \hat{\Theta}_n^{r,e}\|^2 + 2\eta^2\sum_{m \notin \mathcal{C}_{\boldsymbol{s}_{e+1}}}\|\nabla_{m} f_n(\hat{\Theta}_n^{r,e})\|^2 \nonumber\\
    &  + 2\eta^2\sum_{m \in \mathcal{C}_{\boldsymbol{s}_{e+1}}}\|\nabla_{m} f_n(\Theta_n^{r,e})- \nabla_{m} f_n(\hat{\Theta}_n^{r,e})\|^2 \nonumber\\
    \leq & 2 \|\Theta_n^{r,e} - \hat{\Theta}_n^{r,e}\|^2 + 2 L^2 \eta^2 \sum_{m \in \mathcal{C}_{\boldsymbol{s}_{e+1}}}\|\Theta_n^{r,e}- \hat{\Theta}_n^{r,e}\|^2 \nonumber\\
    & + 2\eta^2\sum_{m \notin \mathcal{C}_{\boldsymbol{s}_{e+1}}}\|\nabla_{m} f_n(\hat{\Theta}_n^{r,e})\|^2 \nonumber\\
    \leq &(2 + 2 \eta^2L^2|\mathcal{C}_{\boldsymbol{s}_{e+1}}|)\|\Theta_n^{r,e}- \hat{\Theta}_n^{r,e}\|^2 + 2\eta^2| M -\mathcal{C}_{\boldsymbol{s}_{e+1}}|\delta^2 
\end{align}

Thus, we can have:
\begin{align}
    & \|\Theta_n^{r,{E_r} } - \hat{\Theta}_n^{r,{E_r} }\|^2 \nonumber\\
     \leq & 2\eta^2 \sum_{e=1}^{E_r}  \left( \prod_{j=e}^{{E_r} -1} (2 + 2 \eta^2 L^2 |\mathcal{C}_{\boldsymbol{s}_j}|) \right) (M - |\mathcal{C}_{\boldsymbol{s}_e}|) \delta^2
\end{align}

\end{proof}

To achieve the tightest bound on the divergence between the actual and optimal global models, we must strategically adjust the training order to minimize the sum on the right-hand side of the inequality. 

\begin{Corollary}
Selecting \(|\mathcal{C}_{\boldsymbol{s}_e}|\) in descending order makes the bound in Theorem \ref{thm1} the tightest.
\end{Corollary}

\begin{proof}
The corollary can be proved using the steps of an induction proof.
Consider the term for \(e=1\):
\begin{align}
\prod_{j=1}^{{E_r} -1} (2 + 2 \eta^2 L^2 |\mathcal{C}_{\boldsymbol{s}_j}|)
\end{align}
This product factor is the same regardless of the order of \(|\mathcal{C}_{\boldsymbol{s}_j}|\). The second factor in the term for \(e=1\) is $
(M - |\mathcal{C}_{\boldsymbol{s}_1}|)$
To minimize this, we should choose the largest \(|\mathcal{C}_{\boldsymbol{s}_1}|\) since \(M - |\mathcal{C}_{\boldsymbol{s}_1}|\) will then be the smallest. Besides, as the largest \(|\mathcal{C}_{\boldsymbol{s}_1}|\) has been selected, the product factor in the second term is lower than the scenario where \(|\mathcal{C}_{\boldsymbol{s}_1}|\) is not the largest.

We make the inductive hypothesis that selecting the \(k\) largest \(|\mathcal{C}_{\boldsymbol{s}_j}|\) in descending order minimizes both the sum of the first \(k\) terms and the product factor in the \((k+1)\)-th term.

Based on the inductive hypothesis, the product factor in the \((k+1)\)-th term is minimized and is not affected by the selection of \(|\mathcal{C}_{\boldsymbol{s}_{k+1}}|\). To minimize the \((k+1)\)-th term, we should choose the largest remaining \(|\mathcal{C}_{\boldsymbol{s}_{k+1}}|\). This ensures that $(M - |\mathcal{C}_{\boldsymbol{s}_{k+1}}|)$ is minimized. As the selection of \(|\mathcal{C}_{\boldsymbol{s}_{k+1}}|\) will not change the product term in \((k+1)\), the overall \((k+1)\)-th term is minimized. Thus, we have the sum of the \((k+1)\) terms minimized by selecting the largest in descending order. Additionally, since the first \(k+1\) largest \(|\mathcal{C}_{\boldsymbol{s}_j}|\) have been selected, it minimizes the product term in the \((k+2)\)-th term.

By recursively applying this selection strategy, where we always choose the largest remaining \(|\mathcal{C}_{\boldsymbol{s}_j}|\), we ensure that each step minimizes the bound incrementally, leading to the tightest possible bound for the entire expression.

\end{proof}

\textbf{Remark:} By sorting the cardinality of modality combinations \(|\mathcal{C}_{\boldsymbol{s}_e}|\) in descending order, we begin the training process with combinations that contain a larger number of modalities. This approach not only tightens the bound but also aligns with the intuitive understanding that initiating training with more comprehensive modality combinations can reduce the staleness information obtained by each encoder.



\section{Experiments}
\label{sec:exp}
All experiments were conducted using federated learning simulations on the PyTorch framework and trained on Geforce RTX 3080 GPUs. Each experiment was repeated five times with different random seeds, and the averaged results are reported. The detailed experimental settings are provided below.

\subsection{Experiments Setup}
\subsubsection{Dataset}

\textbf{UCI-HAR.} 
The UCI-HAR dataset \cite{UCI} is a well-established dataset for human activity recognition. It comprises data from 30 participants (average age: 24) performing six activities: walking, walking upstairs, walking downstairs, sitting, standing, and lying down. These activities are recorded using smartphone sensors, specifically accelerometers and gyroscopes, which capture three-dimensional motion data. The data are sampled at a rate of 50 Hz, resulting in 128 readings per window per sensor axis.

\textbf{KU-HAR.} 
The KU-HAR dataset \cite{sikder2021ku} is a more recent human activity recognition dataset collected from 90 participants (75 male and 15 female) performing 18 different activities. For this study, we focused on the six activities that overlap with those in the UCI-HAR dataset. In KU-HAR, each time-domain subsample includes 300 time steps per axis for both accelerometer and gyroscope data. 

\textbf{MVSA-Single.}
The MVSA-Single dataset \cite{niu2016sentiment} is designed for research in multimodal sentiment analysis, capturing both textual and visual cues from social media. Each entry in the dataset consists of a single image paired with its corresponding textual content, annotated to reflect the expressed sentiment. The dataset includes various sentiment labels that categorize the emotional tone conveyed by the both text and image. 

\subsubsection{Encoder and Header Model}
For the UCI-HAR and KU-HAR datasets, which have two modalities (accelerometers and gyroscopes), we use a CNN-based model as the encoder for the accelerometer data. This model consists of five convolutional layers and one fully connected layer. For the gyroscope data, we use an LSTM-based model with one LSTM layer and one fully connected layer. Both encoders produce an output of length 128. The header model comprises two fully connected layers.

For the MVSA-Single dataset, which includes text and image modalities, we use four layers CNN with a modified output layer to produce an output dimension of 128 as the image encoder. For the text encoder, we use a two layers LSTM model with an output of 128. The header model for this dataset also consists of two fully connected layers.

\subsubsection{FL Training Setup}
For the UCI-HAR and KU-HAR datasets, we consider a system with 10 clients, each possessing 600 data samples of both accelerometer and gyroscope modalities. The client data is distributed non-IID with a Dirichlet distribution ($\alpha = 10$) in the main experiment. We also investigate other non-IID degrees. The learning rate is tuned from \{0.005, 0.01, 0.1\}, with a decay factor of 0.99 until it reaches 0.001. The training batch size is set to 64. In the main experiments, the local training GPU time $T$ is set to 24. Based on a rough estimate of the ratio of full training time to each encoder's training time, we round the time to integers, setting the full model training time $t_{\{acc, ~gyro\}}$ to 5, the accelerometer encoder training time $t_{acc}$ to 4, and the gyroscope encoder training time $t_{gyro}$ to 3. We also test different values of $T$ in the experiments.

For the MVSA-Single dataset, we consider a system with 20 clients, each having 300 data samples with both text and image modalities. The client data is non-IID with a Dirichlet distribution ($\alpha = 10$). The learning rate is tuned from \{0.001, 0.01, 0.1\}, with a decay factor of 0.99 until it reaches 0.0001. The training batch size is set to 64. The local training GPU time $T$ is set to 24. Based on a rough estimate of the ratio of full training time to each encoder's training time, we round the time to integers, setting the full model training time $t_{\{image, ~text\}}$ to 5, the image encoder training time $t_{image}$ to 3, and the text encoder training time $t_{text}$ to 3.

\subsubsection{DDPG Training Setup}
We train the DDPG agent on different datasets to achieve the best performance for each specific task. The Actor network in the DDPG model is designed with three fully connected layers, each containing 64 neurons. The input size of the Actor network corresponds to the state size, and the output size matches the action size, where each action passes through a softmax layer to produce a probability distribution. The Critic network also comprises three fully connected layers with 64 neurons each, taking the combined state and action as input and outputting a single Q-value. The DDPG model employs the Adam optimizer with a learning rate of 0.0001 for both the Actor and Critic networks. The soft target updates for both networks are performed with a tau value of 0.001 to ensure stability. The target accuracy $\text{Acc}^{\text{target}}$ is set to 0.68 for the training on the UCI-HAR dataset, 0.68 for the training on KU-HAR dataset and 0.7 for the training on MVSA-Single dataset.

\subsubsection{Baselines} 
To the best of our knowledge, our paper is the first to adaptively allocate the training frequencies of each modality encoder in the MFL setting. Consequently, there are no direct methods available for comparison. To validate the efficiency of our proposed method, FlexMod, we consider the following baselines: \textbf{Entire Update:} In this method, all modality encoders are updated simultaneously, meaning the update frequencies of the modality encoders are identical. The local training time constraint $T$ is also applied. We also consider a series of baseline methods that update only a single modality encoder. For the UCI-HAR/KU-HAR datasets, we use the following two baselines: \textbf{Acce Only:} Only the Accelerometer encoder is updated during each local training. \textbf{Gyro Only:} Only the Gyroscope encoder is updated during each local training. Similarly, for the MVSA-Single dataset, we consider \textbf{Image Only} and \textbf{Text Only} as baseline methods, where only the text or image modality encoder is updated.

\begin{figure}[h]
\centering
 \subfloat[Non-IID ($\alpha = 1$)]{\includegraphics[width=0.495\linewidth]{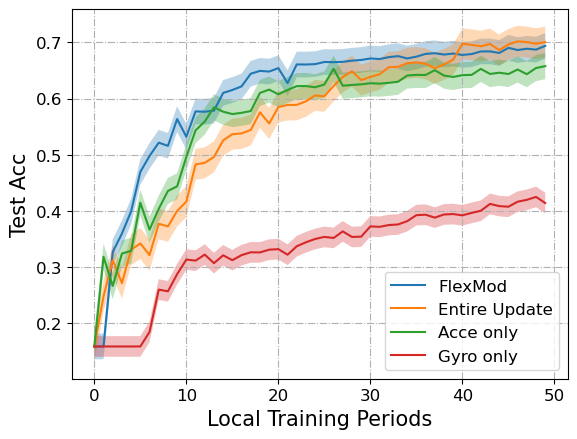}} 
\subfloat[Non-IID ($\alpha = 10$)]{\includegraphics[width=0.495\linewidth]{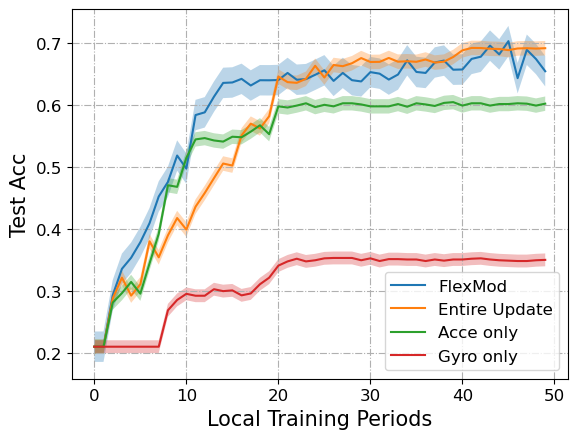}} 
\caption{Convergence performances on UCI-HAR} \label{UCI-perf}
\end{figure}

\subsection{Experiment Results}
The primary objective of the experiments is to demonstrate the differences in the number of global training rounds required by various methods to achieve a specific test accuracy, thereby highlighting the difference in convergence speeds.

\textbf{Performance Comparison.} 
We first compare the convergence performance of our proposed method with baseline methods on the UCI-HAR dataset, focusing on the more realistic non-IID setting compared to the IID setting. We consider two levels of non-IID distribution (i.e., Dirichlet $\alpha = 1/10$). As shown in Fig. \ref{UCI-perf}, for both non-IID settings, our proposed methods outperform the baseline methods in terms of convergence speed. It is important to note that with a sufficiently large number of global training rounds, \textbf{Entire Update} should achieve the highest accuracy, as it avoids any staleness in updating the modality encoder. However, it requires a significantly longer time to achieve this accuracy. Our proposed method can reach comparable test accuracy to \textbf{Entire Update} but with a much faster convergence speed. For example, in Fig. \ref{UCI-perf}(b), to achieve a test accuracy of 0.6, our method requires 12 rounds, compared to the 20 rounds required by \textbf{Entire Update}, demonstrating the faster speed of our proposed method. Compared with the other two baselines (\textbf{Acce Only} and \textbf{Gyro Only}), our method shows better performance in both convergence speed and final test accuracy. \textbf{Gyro Only} has a low final accuracy due to the lesser importance of this modality for this classification task. However, the higher test accuracy of \textbf{Entire Update} and our proposed method compared to \textbf{Acce Only} indicates that incorporating this modality is beneficial.

\begin{figure}[h]
    \centering
    \begin{minipage}[t]{0.49\linewidth}
	\includegraphics[width=1\linewidth]{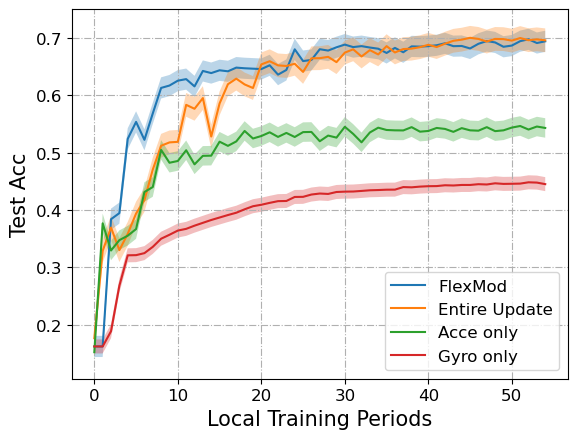}
	\caption{\textcolor{black}{KU-HAR Non-IID}} \label{fig:KU-perf}
    \end{minipage}
     \begin{minipage}[t]{0.49\linewidth}
	\includegraphics[width=1\linewidth]{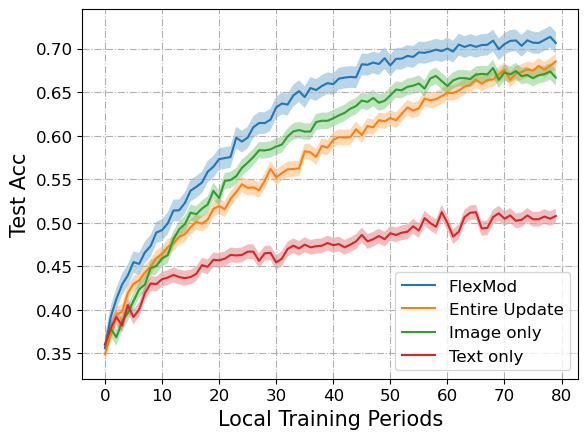}
	\caption{MVSA-Single Non-IID} \label{fig:MVSA-perf}
    \end{minipage}
\end{figure}

We then evaluate our proposed method on the KU-HAR and MVSA-Single datasets under non-IID settings. Similar results to those on the UCI-HAR dataset are observed in Figs. \ref{fig:KU-perf} and \ref{fig:MVSA-perf}. Our proposed method achieves the best convergence performance, with the highest convergence speed while maintaining a comparable final test accuracy to \textbf{Entire Update}.

\textbf{Impact of time constraint $T$.}
In the main experiments, we set the local training time constraint $T$ to be 24. $T = 24$ cannot be fully divided by $t_{\{acc, ~gyro\}} = 5$, potentially causing wasted training time. Here, we set $T = 15$, which allows full division. The results in Fig. \ref{fig:impactT1} show that the improvement of our method slightly decreases compared to the main experiment results in Fig. \ref{UCI-perf} because the \textbf{Entire Update} method can now fully utilize all training resources. However, our method still achieves better performance compared with the baselines. Additionally, fully utilizing local training time is an advantage of our proposed flexible modality encoder training schedule. If the entire update training cannot fit within the remaining time, the client can only be idle, wasting computational resources.

We also consider an extreme condition by setting $T = 4$ on the KU-HAR dataset. With such a small $T$, the local training time is insufficient for even one complete model update, rendering the \textbf{Entire Update} method unavailable. Compared to the other two baseline methods, which update only one modality encoder throughout the training process, Fig. \ref{fig:impactT2} shows that our proposed method, which flexibly trains the modality encoder based on utility, achieves better performance.

\begin{figure}[h]
    \centering
    \begin{minipage}[t]{0.49\linewidth}
	\includegraphics[width=1\linewidth]{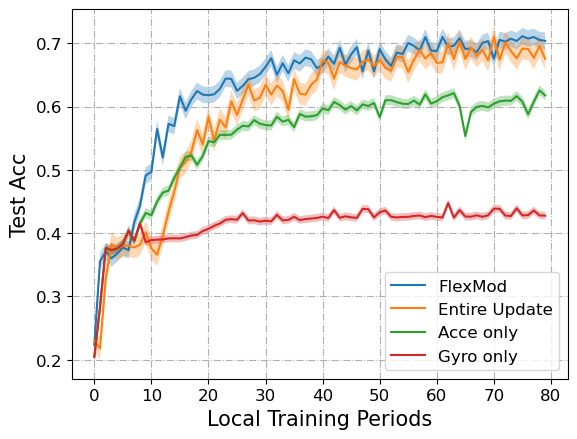}
	\caption{KU-HAR (T = 15)} \label{fig:impactT1}
    \end{minipage}
     \begin{minipage}[t]{0.49\linewidth}
	\includegraphics[width=1\linewidth]{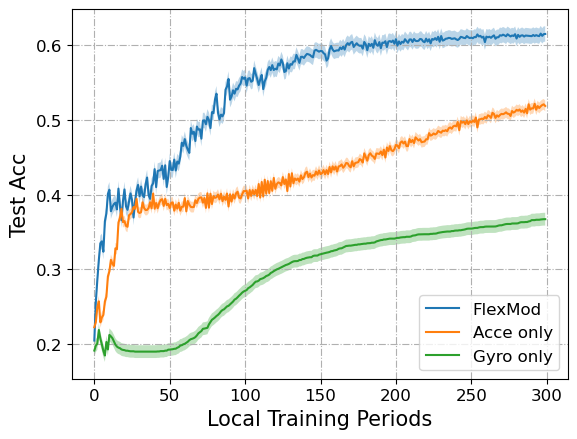}
	\caption{KU-HAR (T = 4)} \label{fig:impactT2}
    \end{minipage}
\end{figure}

\textbf{Combination level decision.} In our proposed method, we determine the training frequencies at the combination level rather than the modality level. As shown in Fig. \ref{motivation} (b), training two modality encoders simultaneously (i.e., as a combination) reduces the total training time compared to training each encoder separately. Here, we present additional experimental results on the MVSA-single dataset. The results in Fig. \ref{Combination_level}(a) indicate that under the same time constraint $T$, allocating at the combination level exhibits better convergence performance. Additionally, Fig. \ref{Combination_level}(b) reflects two key points. First, it demonstrates that making decisions at the combination level allows for more frequent updates of each encoder. These results further validate the advantage of making allocation decisions at the combination level. Second, it shows that the image encoder is trained more frequently than the text encoder in both combination and modality level allocations, highlighting the differing importance of each encoder and the need to prioritize the more critical modality.

\begin{figure}[h]
\centering
 \subfloat[Convergence Performance]{\includegraphics[width=0.495\linewidth, height=0.33\linewidth]{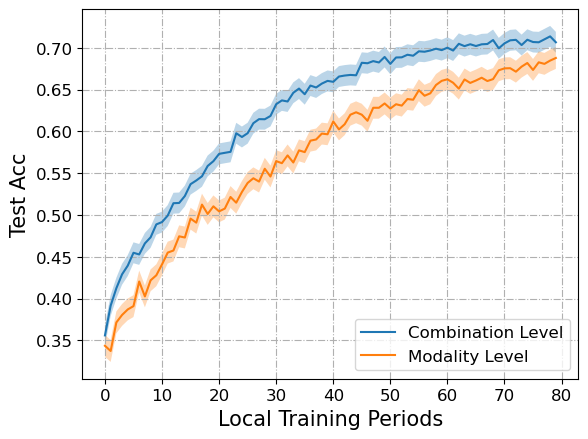}} 
\subfloat[Update Frequencies]{\includegraphics[width=0.495\linewidth, height= 0.33\linewidth]{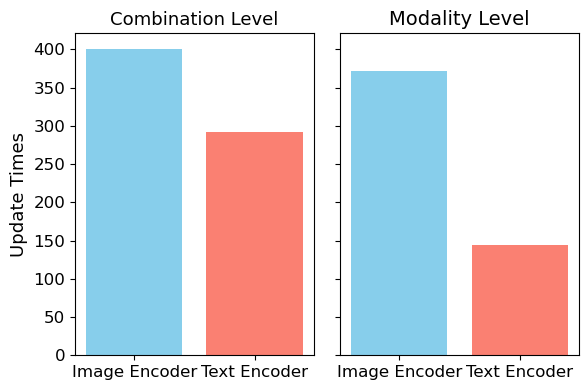}} 
\caption{Impact of Combination Level Decision} \label{Combination_level}
\end{figure}

\begin{figure}[h]
    \centering
    \begin{minipage}[t]{0.49\linewidth}
	\includegraphics[width=1\linewidth, height=0.69\linewidth]{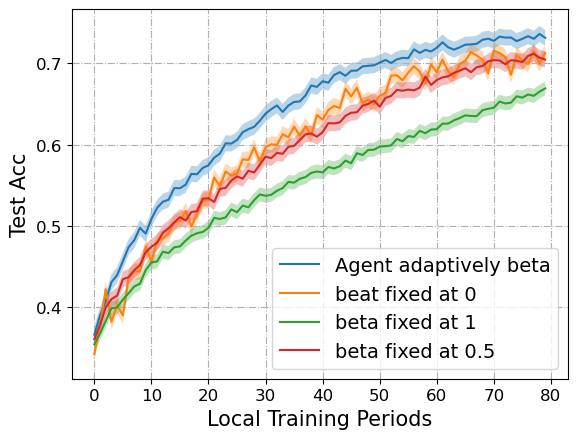}
	\caption{Impact of sum weight.} \label{fig:weight}
    \end{minipage}
     \begin{minipage}[t]{0.49\linewidth}
	\includegraphics[width=1\linewidth, height=0.69\linewidth]{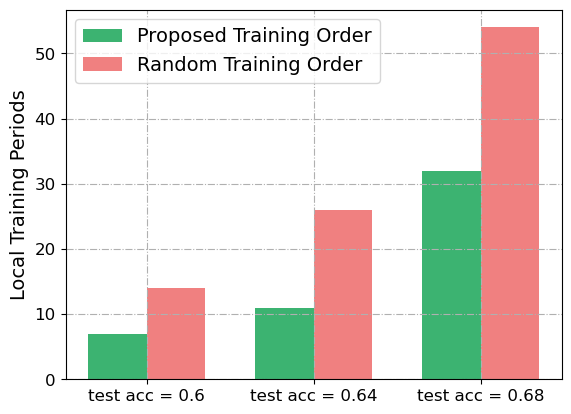}
	\caption{Impact of training order.} \label{fig:order}
    \end{minipage}
\end{figure}

\textbf{Impact of Sum Weight.} In the proposed method, the sum weight \(\beta^r\) is determined by the DDPG agent for each local training period \(r\). Here, we investigate the impact of this value on the MVSA-Single dataset. First, consider two extreme conditions where \(\beta^r\) is fixed at 0 and 1. When \(\beta^r\) is fixed at 0, the utility function considers only the modality importance index in decision-making. Conversely, when it is fixed at 1, the utility function relies solely on the modality encoder quality. The results in Fig. \ref{fig:weight} show that both extreme weight settings perform worse than our proposed method, underscoring the importance of considering both factors in the allocation decision. However, fixing \(\beta^r\) at an arbitrary constant (i.e., 0.5) limits performance improvement and can even perform slightly worse than the extreme cases. This highlights the necessity of adaptively setting the weight and using DDPG to determine the weight based on the specific state of each round.

\textbf{Impact of Training Order.} In this experiment, we investigate the impact of training order on KU-HAR dataset. Based on the insight from Theorem 1, after obtaining the frequencies of each combination during the local training, we train the combinations in descending order based on the number of modalities contained in each combination. Fig. \ref{fig:order} demonstrates that, compared to random order training, the proposed training order achieves the same target test accuracies with fewer local training periods.

\section{Conclusion}
\label{sec:con}
Existing MFL methods often uniformly allocate computational frequencies across all modalities. In this paper, we propose FlexMod, a novel approach to enhance computational efficiency in MFL by adaptively allocating training resources for each modality encoder based on their importance and training requirements. We employ prototype learning to assess the quality of modality encoders, use Shapley values to quantify the importance of each modality, and adopt the DDPG method from deep reinforcement learning to optimize the allocation of training resources. Experimental results on three real-world datasets validate the efficiency of our proposed method. One limitation of this work is that we only consider the full participation scenario in the FL setting. Thus, a future research direction is to consider the client partial participation setting.

\bibliographystyle{IEEEtran}
\bibliography{reference}

\end{document}